\documentclass[11pt]{article}
\usepackage{geometry}
\usepackage{graphicx}
\usepackage{amsmath}
\usepackage{amssymb}
\usepackage{amsfonts}
\usepackage{amsthm}
\usepackage{algorithm2e}
\usepackage{todonotes}
\usepackage{natbib}

\newtheorem{thm}{Theorem}

\newtheorem{prop}[thm]{Proposition}

\newcommand{\R}{\mathbb{R}}

\renewcommand{\and}{\text{ and }}

\newcommand{\E}{\mathbb{E}}

\renewcommand{\P}{\mathbb{P}}
\newcommand\floor[1]{\left\lfloor#1\right\rfloor}

\newcommand\curly[1]{\left\{#1\right\}}
\newcommand\brac[1]{\left[#1\right]}
\newcommand\paren[1]{\left(#1\right)}
\newcommand\abs[1]{\left|#1\right|}

\renewcommand{\hat}{\widehat}
\renewcommand{\tilde}{\widetilde}

\DeclareMathOperator*{\argmax}{arg\,max}

\title{Fair Information Spread on Social Networks with Community Structure}
\author{Octavio Mesner, Elizaveta Levina, Ji Zhu \\ Department of Statistics, University of Michigan}
\date{\today}
\begin{document}

\maketitle

\begin{abstract}
	Information spread through social networks is ubiquitous.
	Influence maximization (IM) algorithms aim to identify  individuals who will generate the greatest spread through the social network if provided with information, and have been largely developed with marketing in mind. In social networks with community structure, which are very common, IM algorithms focused solely on maximizing spread may yield significant disparities in information coverage between communities, which is  problematic in settings such as  public health messaging.
	While some IM algorithms aim to remedy disparity in information coverage using node attributes, none use the empirical community structure within the network itself, which may be beneficial since communities directly affect the spread of information.   Further, the use of empirical network structure allows us to leverage community detection techniques, making it possible to run fair-aware algorithms when there are no relevant node attributes available, or when node attributes do not accurately capture network community structure.
	In contrast to other fair IM algorithms, this work relies on fitting a model to the social network which is then used to determine a seed allocation strategy for optimal fair information spread.   We develop an algorithm to determine optimal seed allocations for expected fair coverage, defined through maximum entropy, provide some theoretical guarantees under appropriate conditions, and demonstrate its empirical accuracy on both simulated and real networks.  Because this algorithm relies on a fitted network model and not on the network directly, it is well-suited for partially observed and noisy social networks.
\end{abstract}

\section{Introduction}

Individuals frequently learn new information or get recommendations from others who have themselves learned it from someone else, creating a cascade of information moving along a social network.
Much of the past work analyzing social networks as a medium for the transmission of information has focused on resource allocation to maximize the total number of individuals within the network who are exposed to a piece of information.

Influence maximization (IM) aims to choose a given number of nodes ({\em seeds}), which, if given  information to distribute ({\em activated}), will pass it to the largest expected number of nodes, or perhaps optimize some other related objective.   This typically  involves probabilistic modeling of information spread.
IM initially appeared as a method for viral marketing, choosing influencers to promote products in order to maximize a company's  revenue~\citep{domingos2001mineNetwork}.
Shortly after, \cite{kempe2003maximizing} developed a mathematical framework for information spread and a method for optimization.
This work established the \emph{independent cascade} model, where the event of information transmission from an activated node to an unactivated node over a network edge is a binary random variable, independent of other transmissions.
This paradigm, which has been widely used, provides a framework for approximating the optimal solution with theoretical lower bound guarantees.
Since then, many variations on the problem and solutions have been proposed;  see survey papers ~\citep{li2018survey, banerjee2020survey, aghaee2021survey}.

Solely focusing on maximizing coverage is reasonable in many settings, for instance in marketing, but in settings where fair access is a concern, for instance in public health messaging, it may lead to significant disparities in access to information depending on one's place within the network, which is often influenced by attributes such as race, gender, LGBTQ status, etc.
Given that many social networks have community structure, 
aiming to reach the greatest number of individuals may push algorithms to allocated resources to larger and/or denser communities in favor of smaller and less well connected ones.
For example, the racial disparity in HIV pre-exposure prophylaxis uptake in the US~\citep{jenness2019prep} may be due, at least in part, to social influence or lack thereof.
LGBTQ disparities in STEM fields and common stereotypes~\citep{freeman2020lgbtq} may also be propagated through social networks.

In general, the goal of fair machine learning is to develop \emph{fair-aware} methods and algorithms whose performance remains consistent regardless of sensitive attributes such as race or gender. 
This is often accomplished by augmenting a prediction loss in an optimization problem with some additional metric of fairness, particularly with respect to known disadvantaged communities~\citep{barocas2017fairness}.
Fair-aware IM algorithms typically augment an IM objective function with a fairness penalty to reduce potential disparity between network communities, frequently in information coverage after information spread but sometimes in seed allocation as well.
For instance, \cite{stoica2019fairness} evaluate fairness in both seed placement and information coverage after information spread.
Using the independent cascade model, they employ the greedy seed selection approach from \cite{kempe2003maximizing} and a high-degree node heuristic, and find that the greedy approach is roughly fair in terms of seed placement but not information coverage after spread, while the heuristic is not fair for either.
The paper argues that, at least in smaller networks, a fair-aware approach 
can have similar results in overall coverage as IM while also achieving greater fairness in seeding and coverage, and uses a fairness constraint requiring seeds be allocated to both communities.  This paper does not consider the case of more than two communities.  
\cite{tsang2019fairness} quantifies fairness of information coverage after spread using a game-theoretic approach by applying the concept of individual rationality to groups: ``no group should be better off by leaving the (IM) game with their proportional allocation of resources and allocating them internally''.
Using this principle, the paper develops an algorithm for choosing seed nodes to fairly distribute information across different populations, particularly with respect to sensitive attributes such as race and/or sex, while also maximizing the total coverage.
This paper uses data on homeless youth messaging for HIV prevention.
In this setting, it is well documented that racial and sexual minorities represent the greatest risk of HIV seroconversion, thus it is a priority for any messaging campaign to reach these communities.
\cite{ali2022fairness} also consider fairness in coverage after information spread but in time-critical settings, such as job announcements, again using an independent cascade model of information spread.
The paper has two objectives:
First, it attempts to achieve information coverage within a set period of time given a seed allocation budget.
Second, it attempts to meet a coverage quota using the minimum number of seeds.
In both cases, it quantifies fairness using the maximum difference in proportion of coverage between groups and minimizes this value.
\cite{farnad2020unifyingFair} consider the IM problem with fairness in terms of four separate constraint metrics: equality (seed allocation proportional to community size), equity (expected coverage is proportional to community size), maximin (maximize proportional coverage in community with the least proportional coverage), and diversity (each community's coverage must be at least what its internal coverage would be for its allocated number of seeds).
The paper defines each fairness constraint metric in the context of a mixed integer optimization problem with the objective of activating the maximum number of nodes.
\cite{becker2022fairness} takes a more probabilistic approach than prior methods, attempting to maximize the objective over a space of probability distributions for choosing seeds.
That is, this work does not consider fixed seed allocations, but distributions for generating allocations.
The authors show that this approach can achieve the approximation upper bound found in~\cite{kempe2003maximizing} within the fairness paradigm.

\subsection{Our Contribution}

Prior work on fair influence maximization typically assumes that social networks and community membership are fully observed and non-random.
Unfortunately, these assumption may not always be realistic.
For example, an individual may learn some piece of information from an observed social network such as Twitter or Facebook then share it with a new acquaintance via another, unobserved medium, such as text message or word of mouth.
Because information can travel over any social network, observed or not, it may be more realistic to consider an observed social network as a sample from a larger, unseen network.
Further, social networks are commonly observed to have \emph{community structure}, typically interpreted as groups of nodes that are more likely to be connected to each other than to others~\citep{girvan2002community}.
Previous work on fair information spread has primarily used observed node attributes to determine community membership and resource allocation.
Unfortunately, node attributes are not always available, and may be protected by privacy considerations; and even when they are available, the degree to which they correspond to the existing community structure within the network may vary.


Fairness is often focused on sensitive attributes, such as  LGBTQ or HIV status, many of which are likely to be reported inaccurately, adding noise.
Moreover, node attributes may not correspond to communities in the network sense of closer ties.
\cite{tsang2019fairness} hints at this discrepancy in its discussion and analysis of the \emph{price of fairness}, which is defined to be the decrease in optimal IM coverage caused by imposing a fairness constraint.
In their analysis on price of fairness~\cite[Thm. 4.1]{tsang2019fairness}, the authors show that node attributes not reflecting the empirical community structure may lead to a greater cost of fairness.
In settings with many node attributes, the closeness of association with empirical communities can vary widely. 
Thus, rather than using node attributes, we focus on the network community structure itself, estimating it from the observed network, which is a well-studied problem.   

Here, we begin by assuming that observed social networks are generated from either a stochastic block model (SBM)~\citep{holland1983sbm} or a degree-corrected stochastic block model (DCSBM)~\citep{karrer2011stochastic}.
Combining the network model with the independent cascade model for information spread, used by most other fair IM methods, we develop an approximation for the expected number of new activations within each community based on seed node placement determined by the parameters of the fitted model.   This expected coverage is used together with a fairness penalty, quantified through entropy of proportions of activated nodes among communities, to create an objective function that balances coverage with fairness.  
We then maximize this objective function over the space of seed node sets of a given size.

An important aspect of this approach is that the optimization only requires network parameters, but not the network itself.
When a network is available, one can use a community detection method to estimate these parameters, but if a network is not directly observable, which is likely the case in many public health settings, one could also use prior information and expert recommendations for network parameters in place of estimates.

\section{Background and notation}

Assume the network consist of $n$ nodes that are expected to form $K$ communities.  The network is represented by its symmetric binary adjacency matrix $A$, defined by $A_{ij} = 1$ if nodes $i$ and $j$ are connected by an edge and $A_{ij} = 0$ otherwise.
Let $c_i \in \{1, \dots, K\}$ be the community label of node $i$, for $i = 1, \dots, n$, let $C_k :=\curly{i=1,\dots, n: c_i=k}$ be the set of nodes belonging to community $k$, and let $n_k = |C_k|$.  We treat $K$ as known, though in practice it can be estimated by one of the many methods available for estimating the number of communities in a network, such as \cite{le2015betheHessian} or \cite{chen2018network}. 

Let the vector $s\in \curly{0,1}^n$ be the seed vector, with  $s_i=1$ if node $i$ is a seed and zero otherwise.
Let  $q_k^{(t)}(s)$ be the proportion of activated nodes in community $k$ after $t$ time steps of information spread seeded with $s$;  unless it creates ambiguity, we may suppress the dependence on $S$ and/or $t$ in what follows. Let  $q^{(t)} = \brac{ q^{(t)}_1, \dots, q_K^{(t)}}$ be the vector of these proportions for all communities,  and let $p_k^{(t)} :=  (\sum_\ell q_\ell^{(t)})^{-1} q_k^{(t)}$, normalizing to sum up to 1.  
Let $m^{(t)}(s)$ be the total proportion of activated nodes in the network at time $t$, defined by 
$$m^{(t)}(s) = \frac{1}{n} \sum_{k=1}^K n_k q_k^{(t)}.$$
The goal of IM is to determine $S$ to maximize $m(s)$ after some amount of time $t$.  While for some applications taking the limit as $t \rightarrow \infty$ makes sense, for settings where fairness is important, such as public health messaging, timeliness is likely important as well.  Thus we will be looking at spread over a finite number of time points, and behavior at small values of $t$ is likely more important in practice.   

\subsection{Entropy as a measure of fairness}\label{fairEntropy}

There are many ways to quantify fairness; some approaches focus on minimizing disparity in information coverage between groups, while others aim to improve coverage for the worst off.
The Gini coefficient \citep{gini1912variabilita}, frequently used to measure wealth inequality in economics, is defined 
\begin{equation*}
	G = \frac{\sum_{i=1}^n \sum_{j=1}^n \abs{x_i-x_j}}{2\sum_{i=1}^n\sum_{j=1}^n x_j} , 
\end{equation*}
where $x_i$ is the wealth of individual $i$.
A Gini coefficient equal to 0 indicates perfect equality, and the value of 1 indicates that a single individual holds all of the wealth.
\cite{ali2022fairness} uses a related fairness metric which aims to minimize the greatest absolute difference in coverage between groups.   
Both of these measures are not differentiable, which can limit the choice of optimization methods.
The maximum absolute difference metric also only depends on the max and min values across groups, and none of the rest.  
The minimax metric~\citep{rawls1971theory} aim to maximize the benefit among the most disadvantaged.
\cite{farnad2020unifyingFair} and \cite{becker2022fairness} use this metric to improve fairness by maximizing coverage for the group having the least coverage.   A potential drawback of this method is that it does not necessarily minimize disparity between groups.

We use entropy as a fairness metric, though it is more commonly used as a measure of uncertainty.  
For a discrete random variable taking $K$ distinct values, one can show that the uniform random variable placing a probability of $\frac{1}{K}$ on each value has the largest possible entropy in the class.
Entropy is thus a natural measure of fairness, because equal coverage achieves maximum entropy and disparities in coverage lead to a  lower entropy.  Entropy is differentiable on the interior of its domain, simplifying optimization, and in contrast to some other metrics discussed above, all communities contribute to its value.

If $X$ is a discrete random variable over $\curly{1,\dots,K}$ with probability function $\P(X=k) = p_k$, the entropy of $X$ is defined as $$H(X) := -\sum_k p_k \log_K p_k , $$ 
where $0\log 0=0$.
Using $\log_K$ normalizes the entropy so that the maximum achievable value is $1$.
For this application, we normalize $q_k^{(t)}$ so that $p_k^{(t)} = \paren{\sum_j q_j^{(t)}}^{-1} q_k^{(t)}$.

%
\subsection{Network Models}
One of the most popular models for networks with communities is the \emph{stochastic block model} (SBM)~\citep{holland1983sbm};  a recent review can be found in \cite{abbe2017community}.
Under the SBM, community labels $c_i$ for each node $i$ are drawn independently from a multinomial distribution, and the probability of an edge between nodes $i$ and $j$ is determined by the community labels $c_i$ and $c_j$, with edges drawn independently.
For $K$ communities, the SBM is defined by its parameters $\pi = \brac{\pi_1,\dots, \pi_K}$, which is a probability vector ($0 < \pi_k < 1$, $\sum_k \pi_k = 1$) and a $K\times K$ symmetric matrix $P$ with entries in $[0,1]$.  Then we say a network is generated from the $\mathrm{SBM}(n,\pi,P)$ if each $c_i, i = 1, \dots, n$ is drawn independently from a Multinomial($\pi$), and then each $A_{ij}$, $1 \le i \le j \le n$ is drawn independently from Bernoulli($P_{c_i,c_j}$).

The SBM assumptions are likely too simplistic for many real networks, yet in many contexts they provide a good balance of interpretability and model complexity~\citep{zhao2012degreeCorrectedConsistency}.  
One particularly popular generalization of the SBM is the degree-corrected SBM~\citep{karrer2011stochastic}, which removes the assumption of equal expected degrees for all nodes in the same community. In addition to the SBM parameters, the DCSBM also employs node-specific degree parameters $\theta = (\theta_1,\dots, \theta_n)$,  and the probability of an edge  between nodes $i$ and $j$ is given by 
\begin{equation}
	\E[A_{ij}] =\theta_i\theta_jP_{c_i,c_j}~.
 \label{dcsbm}
\end{equation}
For identifiability, we require $\frac{1}{\pi_k n}\sum_i \theta_i I(c_i=k) = 1$ for each community $k$.  
The SBM is a special case of DCSBM obtained by making all $\theta_i$'s equal.  We will present derivations for the more general DCSBM and use both models in our simulation studies.

\subsection{Information Transmission Models}

The most common information transmission models in the literature are the independent cascade (IC) model and the linear threshold model  \cite{kempe2003maximizing}. In the IM and fair IM literature, the independent cascade model is far more common, and it lends itself more easily to analysis; we will assume the IC model here as well.  In the IC model, transmission events are all mutually independent of each other, and a node can only transmit information to its neighbors once, immediately after it receives it (gets activated).
For each pair of connected nodes, $i,j$, let $\beta_{ij}\in [0,1]$ be the probability that node $j$ transmits to node $i$.
Let $S^{(0)} := \curly{i: s_i=1}$ be the set of seed nodes.
Let $S^{(t)} \subset \curly{1,\dots, n}$ be the set of nodes activated by time $t>0$ with $S^{(t-1)} \subseteq S^{(t)}$ and $S^{(t)}\backslash S^{(t-1)}$ be the set of nodes activated at time $t$.
We say that node $i$ is activated at time $t$ if at least one neighbor transmits to $i$ at time $t-1$ and $i$ has not been activated prior to time $t$.


A node can only transmit to a node it is connected with.  Given a network adjacency matrix $A$, let $X_{ij}$ be the event that node $j$ transmits to node $i$ (if node $j$ is activated).  We assume this probability is constant over time, and let  

\begin{equation}
P(X_{ij}=1 | A_{ij} = 1) = \beta_{ij}, 
\label{icmodel}
\end{equation} 
while  $P(X_{ij}=1 | A_{ij} = 0) = 0$.  

To estimate $\beta_{ij}'s$ from data, one has to observe multiple information transmission events, which we do not assume available.    If $\beta_{ij} = \beta$ for all pairs of $(i,j)$, the value of $\beta$ has little bearing on the solution of the IM problem (it only affects the scaling of the tuning parameter which balances IM and fairness). If prior knowledge or expert recommendations are available, they can be easily incorporated into the model, as shown below.  We will assume that $\beta$ is estimated from prior knowledge and treat it as known.

\section{The Fair Information Spread Algorithm}\label{approx}

Our goal is to determine a seed allocation given a fixed budget of seeds $M$, in order to optimize total coverage $m^{(t)}(s)$ while also ensuring fairness.  We do this by adding a fairness penalty quantified by entropy, and solve the optimization problem 
\begin{equation}\label{objective}
	\argmax_{s\in \curly{0,1}^n} m^{(t)}(s) +\lambda H\paren{p_k^{(t)}\paren{s}}
\end{equation}
where $\lambda>0$ is a tuning parameter to be chosen by the user and determines the balance between coverage and fairness.   

Optimizing our objective requires an estimate of the total coverage $m^{(t)}$, which in turn relies on modeling information spread.  
While spread of information through social networks is a familiar everyday experience, its modeling can quickly become complex.
Most methods for modeling information use the observed network as a given.   In contrast, here we model both edge formation (through DCSBM) and subsequent information transmission (through the IC model) to account for noisy, changing, and/or incompletely observed networks.  
While writing down the probability of activation for an individual node is straightforward, estimating spread over time for the entire network is not analytically tractable, and we will rely on recursive approximations to estimate spread.   This estimate will be plugged into the objective function \eqref{objective} to obtain the optimal seed allocation. 

\subsection{Exact spread}\label{sec:exactSpread}

Under our assumption that the network is generated from the DCSBM model \eqref{dcsbm}, the transmissions follow the IC model with probabilities \eqref{icmodel}, and the existence of an edge between nodes and transmissions are independent, the probability that node $j$ will activate node $i$ given that $j$ is currently activated is given by 
\begin{equation}
P(X_{ij} = 1) =  P(X_{ij} = 1|A_{ij} = 1) P(A_{ij} = 1) =  \beta_{ij} \theta_i \theta_j P_{c_ic_j} , 
\end{equation}
This formulation includes the special case of a network generated from the SBM, in which case $P(X_{ij} = 1) = \beta_{ij} P_{c_ic_j}$.

Focusing now on how the activations occur over time, let $G_i^{(t)}=1$ if node $i$ is active at time $t$, and 0 otherwise. Each node can be active only once, and so if $G_i^{(t)}=1$, then $G_i^{(r)}=0$ for $r \neq t$, and $\sum_{r=1}^t G_i^{(r)}=1$ if and only if $i$ has been activated by time $t$. 
The probability that $i$ will be activated by time $t$ is given by Proposition~\ref{eq:exactProb}: 
\begin{prop}\label{eq:exactProb}
    Let $G_i^{(t)}=1$ if node $i$ is active at time $t$, and 0 otherwise.
    Then
    \begin{equation*}
	   \P\paren{\sum_{r=1}^t G_i^{(r)}=1}
	   = 1- \prod_{r=0}^{t-1} \prod_{j\neq i} \brac{ 1- \theta_i \theta_j \beta_{ij} P_{c_ic_j} \P\paren{G_j^{(r)}=1}}~.
    \end{equation*}
\end{prop}
\begin{proof}
    Let $F_i^{(t)}=1$ if any node transmits to $i$ at time $t$, regardless of activation status of $i$, and 0 otherwise.
	Because a node only becomes activated once, $\sum_{r=0}^{t} G_j^{(r)}=1$ if activated by time $t$.
	Then, using De Morgan's law and the independence between transmissions, the probability of node $i$ becoming activated by time $t$ is
	\begin{align*}
		\P&\paren{\sum_{r=1}^t G_i^{(r)}=1}
		  =\P\paren{\bigcup_{r=1}^t \paren{F_i^{(r)}=1}} 
		  = 1- \prod_{r=1}^t \P\paren{F_i^{(r)}=0} \\
		  &= 1- \prod_{r=0}^{t-1} \prod_{j\neq i} \E\brac{ \paren{1-\theta_i \theta_j \beta_{ij} P_{c_ic_j}}^{G_j^{(r)}}}
		  = 1- \prod_{r=0}^{t-1} \prod_{j\neq i} \brac{ 1- \theta_i \theta_j \beta_{ij} P_{c_ic_j} \P\paren{G_j^{(r)}=1}}~.
	\end{align*}
\end{proof}

As one might expect, the probability of a node being activated by time $t$ depends on activations of all other nodes prior to time $t$.  While this formula provides a simple closed-form recursive expression, getting a general closed-form solution that depends only on the seed nodes is complicated, and not likely to yield a function that one can easily work with in an optimization algorithm (except numerically).  

\subsection{Approximate spread}\label{sec:approxSpread}

Rather than working with the exact expression from Proposition~\ref{eq:exactProb} directly, we approximate it to obtain a function that lends itself more easily to optimization.  Applying a first-order Taylor expansion, we have 
\begin{align*}
	\P\paren{\sum_{r=1}^t G_i^{(r)} =1}
	&\approx \sum_{r=0}^{t-1} \sum_{j\neq i} \beta_{ij} \theta_i \theta_j P_{c_ic_j} \P\paren{G_j^{(r)}=1} \\
& = \sum_{j=1}^n \beta_{ij} \theta_i \theta_j P_{c_ic_j} I(i\neq j) \P\paren{ \sum_{r=0}^{t-1} G_j^{(r)}=1} ~.
\end{align*}
We use the fact that node activations over time are mutually exclusive events to switch the order of summation.
This approximation simplifies the recursive function into a composition of linear transformations.
Let $\Psi$ be an $n\times n$ matrix with $\Psi_{ij} = I(i\neq j) \beta_{ij} \theta_i \theta_j P_{c_ic_j}$ where $\Psi_{i\cdot}$ is the $i$th row of $\Psi$.
The corresponding approximation for the total expected number of activations by time $t$ is
\begin{align*}
	\E&\brac{ \sum_{i=1}^n \sum_{r=0}^{t} G_i^{(r)}}
	\approx \sum_{i=1}^n \sum_{j=1}^n I(i\neq j) \beta_{ij} \theta_i \theta_j P_{c_ic_j} \P\paren{ \sum_{r=0}^{t-1} G_j^{(r)}=1} \\
    &\approx \sum_{i=1}^n \sum_{j=1}^n \sum_{k=1}^n  \brac{ I(i\neq j) \beta_{ij} \theta_i \theta_j P_{c_ic_j}}  \brac{I(j\neq k) \beta_{jk} \theta_i \theta_j P_{c_jc_k} } \P\paren{ \sum_{r=0}^{t-2} G_j^{(r)}=1} \\
    &\approx 1_n^T \Psi^t G^{(0)} = 1_n^T \Psi^t s
\end{align*}
where $G^{(0)}=s$ corresponds to seed nodes.
The last approximation applies the Taylor approximation $t-1$ times recursively to express the number of activations by time $t$ as a function of the seed set.
Similarly, the number of expected activations in community $k$ by time $t$ can be approximated by 
\begin{equation*}
	\E\brac{ \sum_{i: c_i=k} \sum_{r=0}^{t} G_i^{(r)}} \approx Z_{\cdot k}^T \Psi^t s
\end{equation*}
where $Z$ is an $n\times k$ membership matrix with entries indicating community, $Z_{ik} = I(c_i=k)$, and $Z_{\cdot k}$ is the $k$th column of $Z$.

While these approximations greatly simplify calculations, they introduce error that accumulates over time.
The exact expression from Proposition~\ref{eq:exactProb} accounts for the possibility of multiple incoming transmissions and dependence of activation status for individual nodes over time. 
The approximation, in contrast, assumes that nodes are only ever activated by single incoming transmissions and that the activation status at different points in time of a single node are independent.
In effect, the approximation overestimates the probability of node $i$ becoming activated by time $t$ and all subsequent approximations (this will become clear from the proof of Proposition~\ref{lem:error}).
Proposition~\ref{lem:error} shows the error incurred due to this approximation as a function of the number of seed nodes and the rate of convergence of network parameters as the number of nodes increases.  

\begin{prop} \label{lem:error}
    If $\theta_i \theta_j \beta_{ij} P_{c_ic_j} =\mathcal O\paren{n^{-\alpha}}$ for $\alpha \in (0,1)$ and $M$ nodes chosen as seeds at time $0$, then 
    \begin{equation*}
	\abs{ \E\brac{ \sum_{i=1}^n \sum_{r=0}^{t} G_i^{(r)}} - 1_n^T \Psi^t \brac{G_j^{(0)}}_j } =\mathcal O\paren{n^{2(1-\alpha)t-1}} M^2~.
    \end{equation*}
\end{prop}
\begin{proof}
    Let $f:[0,1]^d\rightarrow \R$ be defined as $f(x)=1-\prod_{k=1}^d (1-x_k)$.
    Using $\nabla f(x)=\brac{\prod_{k\neq j} (1-x_k)}_j$, the first-order Taylor polynomial of $f$ centered at $0_d$ is $T_1(x) = \sum_{k=1}^d x_k$.
    Because $\nabla^2 f(x) = \brac{ -I(i\neq j) \prod_{k\neq i,j} (1-x_k)}_{ij}$, then for some $c\in[0,1]$, the remainder,
    \begin{equation*}
        R_1(x) = \frac{1}{2}\sum_{i=1}^d \sum_{j=1}^d \brac{-I(i\neq j) \prod_{k\neq i,j} (1-cx_k)} x_i x_j \geq -\frac{1}{2}\sum_{i=1}^d \sum_{j=1}^d I(i\neq j) x_i x_j~.
    \end{equation*}
    By Taylor's theorem, $\abs{ f(x) - T_1(x) } \leq \frac{1}{2}\sum_{i=1}^d \sum_{j=1}^d x_i x_j$.  Then for a node $i$ we have
    \begin{align*}
        &\abs{ \E\brac{ \sum_{r=1}^t G_i^{(r)}} - \sum_{j=1}^n \sum_{r=0}^{t-1} \beta_{ij} \theta_i\theta_j P_{c_ic_j} I(i\neq j) \E\brac{G_j^{(r)}} } \\
        &= \abs{1- \prod_{r=0}^{t-1} \prod_{j\neq i} \brac{ 1- \beta_{ij} \theta_i \theta_j P_{c_ic_j} \E\brac{G_j^{(r)}}} - \sum_{j=1}^n \sum_{r=0}^{t-1} \beta_{ij} \theta_i\theta_j P_{c_ic_j} I(i\neq j) \E\brac{G_j^{(r)}} } \\
        &\leq \frac{1}{2} \sum_{j\neq i} \sum_{k\neq i} \brac{\theta_i \theta_j \beta_{ij} P_{c_ic_j}  \paren{\sum_{r=0}^{t-1} \E\brac{G_j^{(r)}} }} \brac{ \theta_i \theta_k \beta_{ik} P_{c_ic_k}  \paren{\sum_{r=0}^{t-1} \E\brac{G_k^{(r)}}}} \\
        &\leq \mathcal O\paren{ n^{-2\alpha}} \paren{\sum_{j\neq i} \sum_{r=0}^{t-1} \E\brac{G_j^{(r)}}}^2~.
    \end{align*}
	The first inequality above indicates that the approximation error will deviate at most by the sum of the product of probabilities of two incoming transmissions from previously activated nodes.  The second inequality follows from the assumption that $\theta_i \theta_j \beta_{ij} P_{c_ic_j} =\mathcal O\paren{n^{-\alpha}}$, giving a simplified bound on the error for the probability of $i$ being activated by time $t$ as a function of $n$.

	The error bound depends on the number of activated nodes, which is what we are attempting to forecast.
	Using the forecast model itself to approximate $G_j^{(r)}$, we must determine the magnitude of $\Psi_{ij}^{t-1}$.
	Because $\Psi_{ij} = I(i\neq j) \beta_{ij} \theta_i \theta_j P_{c_ic_j} =\mathcal O\paren{ n^{-\alpha}}$, we have 
	\begin{equation*}
		\Psi_{ij}^2 = \sum_{k=1}^n \Psi_{ik} \Psi_{kj} = \mathcal O\paren{ n^{1-2\alpha} }~.
	\end{equation*}
	By induction, we have $\Psi_{ij}^{t} =\mathcal O\paren{n^{(1-\alpha)t-1}}$.
    Then, because $f(x)-T_1(x) = R_1(x)\leq 0$ for $x\in [0,1]^d$, 
    \begin{equation*}
        \sum_{j\neq i} \sum_{r=0}^{t-1} \E\brac{G_j^{(r)}} \leq 1_n^T \Psi^{t-1} \brac{ G_j^{(0)}}_j
        \leq \mathcal O\paren{ n^{(1-\alpha)(t-1) }} M~.
    \end{equation*}
	
    Using the triangle inequality and the error bounds above,
    \begin{align*}
	&\abs{ \E\brac{ \sum_{i=1}^n \sum_{r=0}^{t} G_i^{(r)}} - 1_n^T \Psi^t \brac{G_j^{(0)}}_j } \\
    &\leq \sum_{i=1}^n \abs{ \E\brac{ \sum_{r=1}^t G_i^{(r)}} - \sum_{j=1}^n \sum_{r=0}^{t-1} \beta_{ij} \theta_i\theta_j P_{c_ic_j} I(i\neq j) \E\brac{G_j^{(r)}} } \\
    &\leq \sum_{i=1}^n \mathcal O\paren{n^{-2\alpha}} \brac{ \mathcal O\paren{ n^{(1-\alpha)(t-1) }} M }^2 \\
    &=\mathcal O\paren{n^{2(1-\alpha)t-1}} M^2~.
    \end{align*}
\end{proof}

\subsection{Approximating the objective function}

Next, we use the approximation to the expected spread to approximate the objective function~\eqref{objective}, to obtain a tractable optimization problem.  Specifically, we approximate the objective as a function of the seed node set given the DCSBM parameters, which will be estimated from data.  

Let $s\in \{0,1\}^n$ be the vector indicating initial activation status for each node; that is $s_i=G_i^{(0)}$ for $i=1,\dots, n$.
We can then estimate the proportion of activations within the entire network by time $t$ using the approximation of $\frac{1}{n}\E\brac{ \sum_{i=1}^n \sum_{r=0}^t G_i^{(t)} }$ from Section~\ref{sec:approxSpread} as
\begin{equation*}
	\tilde m^{(t)}(s) := \frac{1_n^T \Psi^t s}{n}~.
\end{equation*}
Similarly, we estimate the proportion of activations in community $k$ by time $t$ as
\begin{equation*}
	\tilde q_k^{(t)}(s) := \frac{ Z_{\cdot k}^T \Psi^t s}{\pi_k n}~, 
\end{equation*}
and rescale it to sum to 1 to obtain $\tilde p_k^{(t)}(s)$.  We then approximate the objective function   
 $f^{(t)}(s)$ by 
$$\tilde f^{(t)}s) := H(\tilde p^{(t)}s)) +\lambda \tilde m^{(t)}(s) \ . $$ 
Proposition~\ref{lem:objBounds} confirms that our approximation to the objective function is close to the truth.

\begin{prop} 
    If $\beta_{ij}\theta_i \theta_j P_{c_ic_j} =\mathcal O\paren{n^{-\alpha}}$ for $\alpha\in(0,1)$ and the process is started with $M$ seeds, then 
    \begin{equation*}
	\abs{\tilde f^{(t)}(s) -f^{(t)}(s) } =\mathcal O\paren{  n^{2(1-\alpha)t-2} } M^2~.
    \end{equation*}
    \label{lem:objBounds}
\end{prop}

\begin{proof} 
    Because the number of communities $K$ does not grow with $n$, we may apply the same error bound from Proposition~\ref{lem:error} to $\tilde q_k^{(t)}(s)$ and $\tilde m^{(t)}(s)$:
    \begin{align*}
	   \abs{ \tilde q_k^{(t)}(s) -q_k^{(t)}(s) } &  = \frac{1}{n} \mathcal O\paren{ n^{2(1-\alpha)t-1}} M^2 = \mathcal O\paren{ n^{2(1-\alpha)t-2} }M^2 \\
	   \abs{ \tilde m^{(t)}(s) -m^{(t)}(s) }&  = \frac{1}{n} \mathcal O\paren{ n^{2(1-\alpha)t-1}} M^2 = \mathcal O\paren{ n^{2(1-\alpha)t-2} }M^2~.
    \end{align*}
Since the total number of activations is at least $M$, when we rescale $\tilde q^{(t)}$ to sum to 1 to obtain $\tilde p^{(t)}$, the denominator is bounded away from zero.  Thus we also have 
	\begin{equation*}
		\abs{ \tilde p_k^{(t)}(s) -p_k^{(t)}(s) } =\mathcal O\paren{ n^{2(1-\alpha)t-2} } M^2~.
	\end{equation*}
	Using the inequality 
	\begin{equation*}
		\abs{ x \log x -y \log y } \leq 4\sqrt{\abs{x-y}} +\abs{x-y}
	\end{equation*}
	when $x>y>0$, we have 
	\begin{equation*}
		\abs{ H\paren{ \tilde p^{(t)} } -H\paren{ p^{(t)} } } =\mathcal O\paren{ n^{2(1-\alpha)t-2} } M^2~.
	\end{equation*}
	Putting this together, 
	\begin{equation*}
		\abs{\tilde f^{(t)}(s) -f^{(t)}(s) } \leq \abs{ H\paren{ \tilde p^{(t)} } -H\paren{ p^{(t)} }} +\lambda\abs{\tilde m^{(t)}(s) -m^{(t)}(s)} = \mathcal O\paren{ n^{2(1-\alpha)t-2} }M^2~.
	\end{equation*}
\end{proof}

\subsection{The optimization algorithm}

Optimizing the fair IM objective directly over all possible seed sets is a discrete optimization problem.  While one could in principle apply a brute force discrete optimization method such as tabu search~\citep{glover1998tabu}, a continuous relaxation is a standard way to obtain a more computationally efficient solution.
We thus relax the feasible space from  $\{s \in \curly{0,1}^n, \sum_i s_i = M \}$ to $\{s \in \brac{0,1}^n, \sum_i s_i = M \}$, allowing $s_i$ to take values in the $[0,1]$ interval. 

Both the DCSBM and especially the SBM assumptions result in stochastically identical nodes, which are indistinguishable in terms of seed allocation.   Under the DCSBM, nodes $i$ and $j$ will be stochasticaly identical if $\theta_i = \theta_j$ and $c_i = c_j$, which reduces to just $c_i = c_j$ for the SBM.   The optimal allocation can only be determined up to the equivalence class of these stochastically identical nodes.  

To remove this ambiguity from the optimization problem, we construct a unique-node, $n\times v$ membership matrix $V$, where $v$ is the number of unique nodes, defined by $V_{ij} = 1$ if node $i$ is stochastically identical to the $j$th unique node, and  $V_{ij} = 0$ otherwise.  In place of $s\in [0,1]^n$, we now optimize over $x\in [0,1]^v$, with fixed unique node multiplicities $w = 1_n^TV$ and the constraints that $w^Tx\leq M$.
 We approximate the number of activations in community $k$ as $Z_{\cdot k}^T \Psi^t V x$, where $x_i=1$ now indicates that the class of nodes identical to the $i$th unique node is chosen to receive seeds. 

Thus, the optimization problem becomes
\begin{align*}
	\argmax_{x\in [0,1]^v} \quad & \tilde f^{(t)}(x) =  \tilde m^{(t)}(x) +\lambda H\paren{ \tilde p^{(t)}(x) } \\
	\text{given} \quad  & w^Tx \leq M.
\end{align*}
 We solve the problem using trust region constrained optimization~\cite{conn2000trust} implemented in Python3 by scipy.   In brief, this method iteratively approximates the second-order Taylor polynomial of the objective function, using the true gradient with a pseudo-Hessian, and optimizes this function within a given radius around the current input.
 
Denote the solution to the optimization problem as $x^* \in [0,1]^v$.
The number of seeds that should be allocated to each class is then given by $y_i = \lfloor w_i x^*_i\rfloor$.  (If it happens that $\sum_i y_i \leq M-1$, we can assign an additional seed to class $j=\argmax_i \curly{w_ix_i^*-y_i}$, and repeat until $\sum_i y_i = M $).

\section{Results on Synthetic Networks}

To empirically evaluate the proposed algorithm, we conducted simulations of information spread over both real and simulated networks.  The focus of comparisons on simulated networks is comparing allocation strategies, and on error accumulation over time.   Since parameter estimation for DCSBM is well understood, we do not focus on evaluating that part of the algorithm in simulations, plugging in the true parameter values into the optimization objective.  

The optimization algorithm outputs the number of seeds to allocate to each unique class of nodes, determined by a unique combination of degree parameter $\theta_i$ and community label $c_i$.  For the SBM, the algorithm simply outputs the number of seeds to allocate to each community.    Whenever the algorithm allocates seeds to a class with more than one stochastically identical node, we assign them randomly within that class. 

Once seeds are allocated, we simulate information spread within the network using the independent cascade model, for a fixed number of transmission steps $t$.    Then we count the number of nodes activated in each community after $t$ steps and calculate proportion of activations in each community. The entire process is repeated 50 times for each set of network parameters, and a new network is generated with the same parameters for every replication. 


We simulate the networks from the SBM rather than the DCSBM here because the alternative allocation strategies are particularly intuitive and easy to illustrate in this case.   We compare four allocation strategies: the  proposed fair IM algorithm (``Proposed"), allocating an equal number of seeds to each community (``Equal"), allocating seeds proportionally to community size (``Proportional"), and allocating all seeds to the largest community (``Largest").

All SBM networks are generated with $n=1000$ nodes, $K=3$ communities, probabilities of community assignment  $\pi^T = \brac{ 0.7, 0.2, 0.1 }$, and the probability matrix 
$$P = \frac{1}{100} \begin{bmatrix} a_1 & 1 & 1 \\ 1 & a_2 & 1 \\ 1 & 1 & a_3 \end{bmatrix} . $$ 
The weight vector $a = \brac{a_1,a_2,a_3}$ varies across our three settings.  
SBM-1 corresponds to  $a=\brac{10, 5, 2.5}$, with larger communities being denser and average expected degree of 56. 
SBM-2 has $a=\brac{2.5, 5, 10}$, with smaller communities being denser and average expected degree of 20.  
Finally, SBM-3 has $a = \brac{5, 5, 5}$, with all communities equally dense and average expected degree of 32.  While the average degrees cannot be made exactly the same if we want to vary how the density relates to community size, they are all comparable from the point of view of SBM network behavior.  

Figure~\ref{nonopt} shows detailed results (both overall and for each community) for the four allocation strategies under SBM-1, and Figure~\ref{nonopt2} compares overall results for the three SBMs.  Both correspond to $t=1$ (single transmission step), $\lambda = 3$, transmission probability $\beta_{ij}=0.2$ for all $i,j$, and a budget of $M = 30$ seeds.  

\begin{figure}[!ht]
	\centering
	\includegraphics[width=0.6\textwidth]{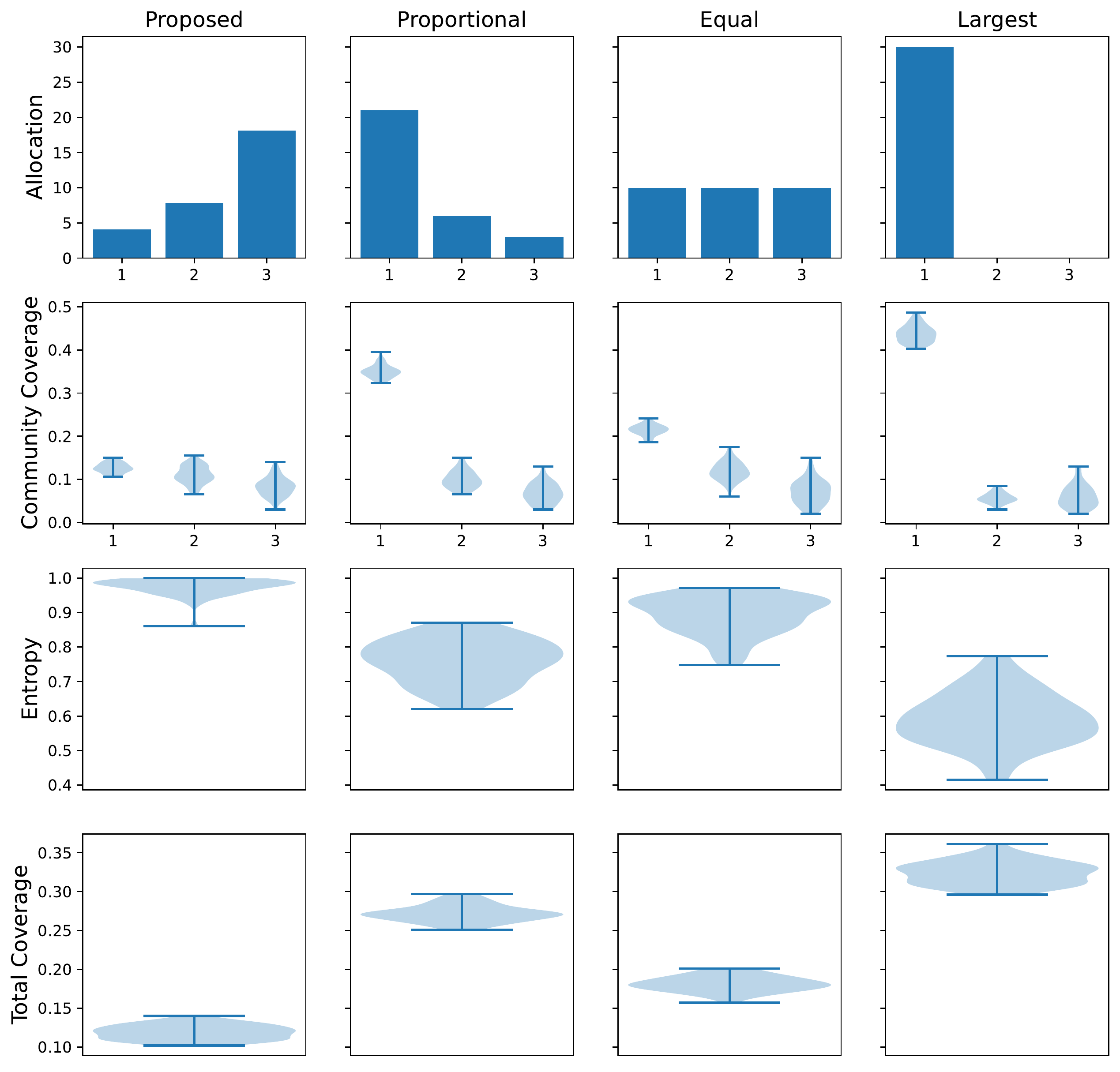}
	\caption{Simulation results on SBM-1 comparing four allocation strategies.  Row 1: number of seeds allocated to each community under different strategies.   Row 2:  Violin plots of coverage for each community, over 50 replications.  Row 3:  violin plots of entropy over communities (fair allocation corresponds to 1).  Row 4: violin plots of overall coverage.  
	}\label{nonopt}
\end{figure}

Figure~\ref{nonopt} shows that under SBM-1 the four different allocation strategies propose quite different allocations and lead to different results.  The proposed method allocates most seeds to the smallest community and the least seeds to the biggest, leading to highest fairness but lowest total coverage in the end.    
In contrast, allocating all seeds to the largest, densest community yields the greatest but least fair coverage.
These results show that, as is often the case, there is a cost to fairness in terms of coverage, which is not surprising:  changing the objective function to include a fairness penalty is likely to lead to a different solution from the one that just maximizes coverage.  

\begin{figure}[!ht]
	\centering
	\includegraphics[width=0.6\textwidth]{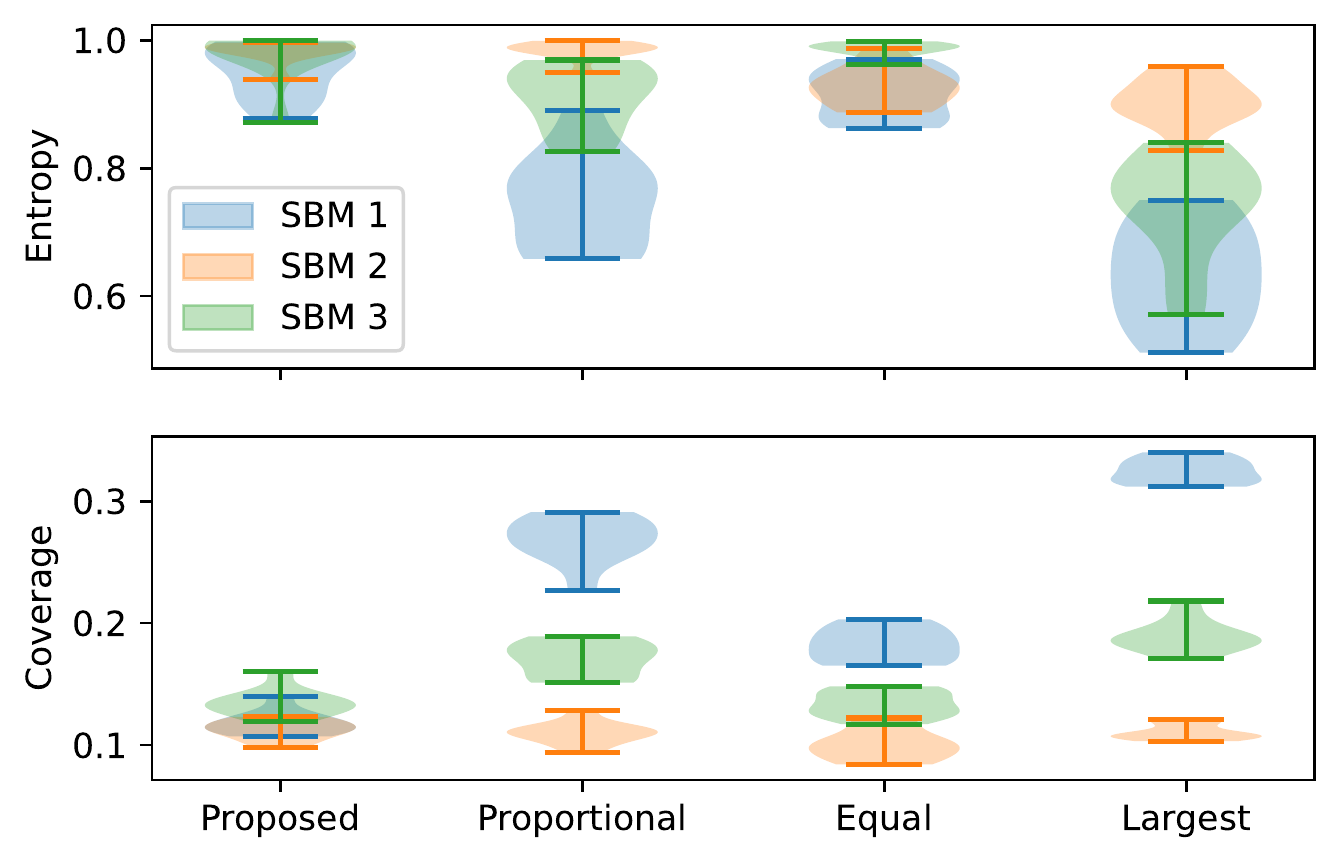}
	\caption{
 Simulation results on SBM-1, SBM-2, and SBM-3 comparing four allocation strategies.  Row 1: violin plots of entropy over communities (fair allocation corresponds to 1).  Row 2: violin plots of overall coverage. 
	}\label{nonopt2}
\end{figure}

Figure~\ref{nonopt2} shows both entropy and coverage for all three SBMs, corresponding to rows 3 and 4 of Figure~\ref{nonopt}, and Table~\ref{tab:fig2alloc} gives the number of seeds allocated to each community by our algorithm.   It is particularly informative to compare SBM-1 and SBM-2.
In SBM-1 the largest community was the densest and smallest community the sparsest, but in SBM-2, the largest community is the sparsest and the smallest is the densest.
In SBM-2, our proposed allocation is very close to proportional.  
Here, allocation strategy seems to have a lesser effect on both entropy and coverage; though allocating all seeds to the largest community does have a slightly lower entropy, all strategies have similar total coverage.   
For SBM-3, our proposed allocation is very close to equal and they have greater entropy with slightly lower total coverage then the other two strategies.   This trade-off is controlled by the value of $\lambda$, which cannot be learned from the data:  it reflects how much weight the practitioner is willing to put on fairness.    

\begin{table}
\begin{center}
\begin{tabular}{cccc}
    Community (\# nodes) & 1 (700) & 2 (200) & 2 (100) \\ 
    \hline
     SBM-1 Allocation &  4 & 8 & 18\\
     SBM-2 Allocation &  20 & 7 & 3\\
     SBM-3 Allocation &  11 & 10 & 9
\end{tabular}
\caption{Seed allocations for each SBM under the proposed algorithm}
\label{tab:fig2alloc}
\end{center}
\end{table}

Next, we consider the effect of choice of $\lambda$ on outcomes.
Figure~\ref{varyinglambda} shows entropy and coverage under SBM-1 (where the allocation strategy has the most effect on the outcome) while varying $\lambda$ from 0.1 to 5.   At smaller values of $\lambda$, the algorithm allocates all seeds to the largest community, but as $\lambda$ increases, the allocation converges to the one proposed by our algorithm in Figure~\ref{nonopt}.
Correspondingly, we see the entropy increase and coverage decrease;  both stabilize for $\lambda > 1$.  
SBM-2 and SBM-3 show a similar, though less dramatic, behavior.
This suggests that fair information spread can be guaranteed by $\lambda > 1$, with the exact value having little effect, and a value of $\lambda$ between 0 and 1 can be chosen to achieve a  compromise between coverage and fairness.

\begin{figure}[!ht]
	\centering
	\includegraphics[width=0.65\textwidth]{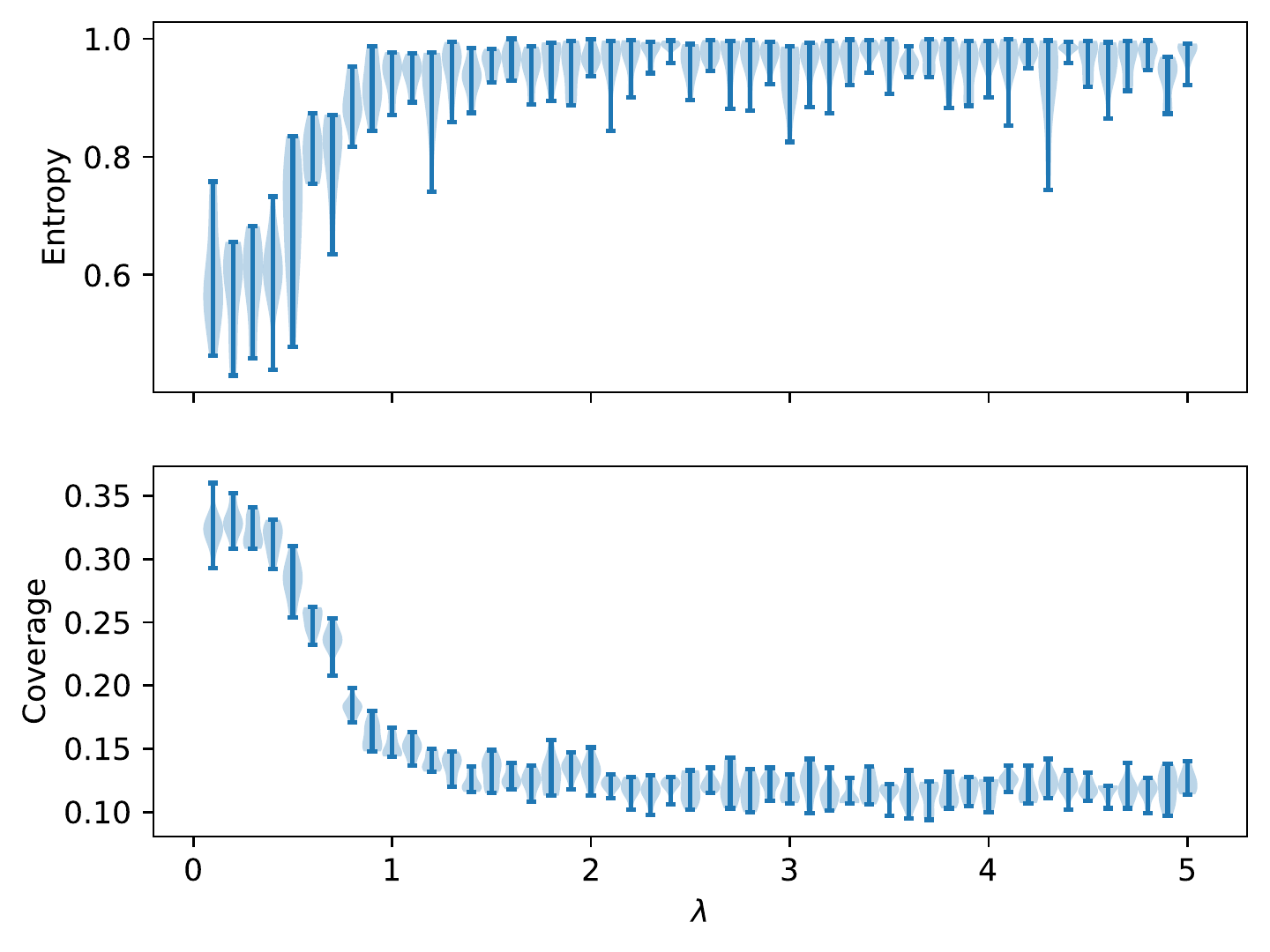}
	\caption{Results for SBM-1 as a function of $\lambda$.  Top row: entropy;  bottom row: coverage.  
	}\label{varyinglambda}
\end{figure}

Finally, we investigate how our approximation behaves over time.  For this simulation, we also take DCSBM (rather than SBM) networks, with $K=3$ communities, community probability vector $\pi^T = \brac{ 0.5, 0.3, 0.2 }$, $n = 1000$ nodes and the probability matrix  
$$P = \frac{1}{100} \begin{bmatrix} 1 & 0.04 & 0.01 \\ 0.04 & 1.2 & 0.05 \\ 0.01 & 0.05 & 1.3 \end{bmatrix}. $$
Degree parameters $\theta_i$ are chosen from a Poisson distribution with mean 5.   We then add 1 to each $\theta_i$ to avoid the possibility of 0 and renormalize to satisfy the identifiability constraint.  The average expected degree under this model is 4.3.
We again set the transmission probability to  $\beta_{ij}=0.2$ for all $i,j$,  choose $\lambda = 3$, and allocate the budget of $M=30$ seeds.

Figure~\ref{dcbmSim} shows the results for three values of time steps, $t = 0$, 3, and 5.  
As the number of time steps increases, our approximations (shown by a red asterisk) get  worse, as we would expect.  However, we see that the allocation itself and the resulting coverage and entropy  change a lot less than our approximations predict, giving us reassurance that the allocation strategy is robust.

\begin{figure}[!ht]
	\centering
	\includegraphics[width=0.65\textwidth]{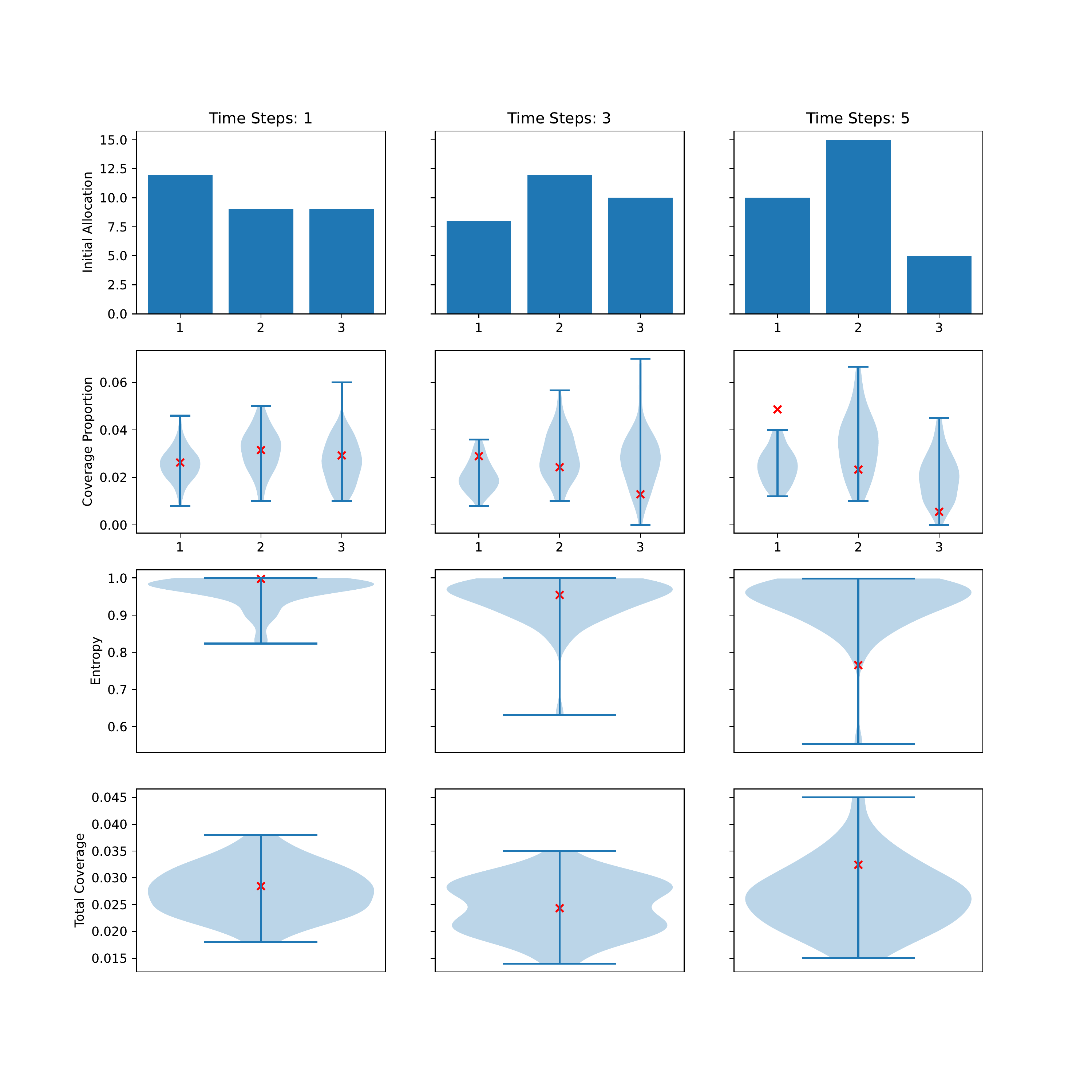}
	\caption{Simulation results over time.  Columns correspond to time steps $t=1, 3, 5$.  Row 1: Proposed seed allocation.   Row 2: Violin plots of coverage proportion for each community.   Row 3: Violin plots of entropy.  Row 4: violin plots of coverage. The red ``X" shows the values predicted based on the approximation. 
	}\label{dcbmSim}
\end{figure}

\section{Results on Real Networks}

In this section, we simulate information spread on real networks.
Data on actual information spread with different allocations is, unfortunately, not available to us, but using real social networks helps create a more realistic assessment of the impact of community structure on information spread.

The political blogs network collected blog websites related to US politics around the 2004 US election \citep{adamic2005political}, and treated webpages as nodes and hyperlinks between them as edges.  We ignore the direction of edges for simplicity. 
Each blog was manually classified as either liberal or conservative, creating two ground-truth communities.
This network exhibits high degree heterogeneity, making it more suitable for DCSBM than SBM as a model for communities.  

We focus on the largest connected component of the network, consisting of 1222 nodes and 16714 edges.
Recognizing the potential noise in the ground-truth communities and the fact that they are often unknown in applications, we estimated community memberships by the SCORE method, which was designed for the DCSBM \citep{jin2015fast}.  
Both ground-truth and estimated communities are of similar size and disagree on less than 5\% of nodes.  
The estimated affinity matrix is $\hat P = \frac{1}{100} \begin{bmatrix} 3.9 & 0.3 \\ 0.3 & 4.5  \end{bmatrix}$.

Because transmission probabilities are also unknown, we varied the probability of transmission within and between communities, each from 0.1 to 0.9 in increments of 0.1.
Though having a between community probability of transmission greater than within community seems unlikely, we still included results from this scenario for completeness.
For all simulations, we used a single time step, $\lambda=3$, and 34 seeds total; we choose $\floor{\sqrt n}$ for the number of seeds.
Figure~\ref{polblog} shows the mean entropy and coverage from 50 simulations of each between/within community probability of transmission scenario.
As one can see, mean entropy remains high, above 0.9, for all settings, and coverage tends to increase with the probability of transmission, particularly within community.


\begin{figure}[!ht]
	\centering
	\includegraphics[scale=0.8]{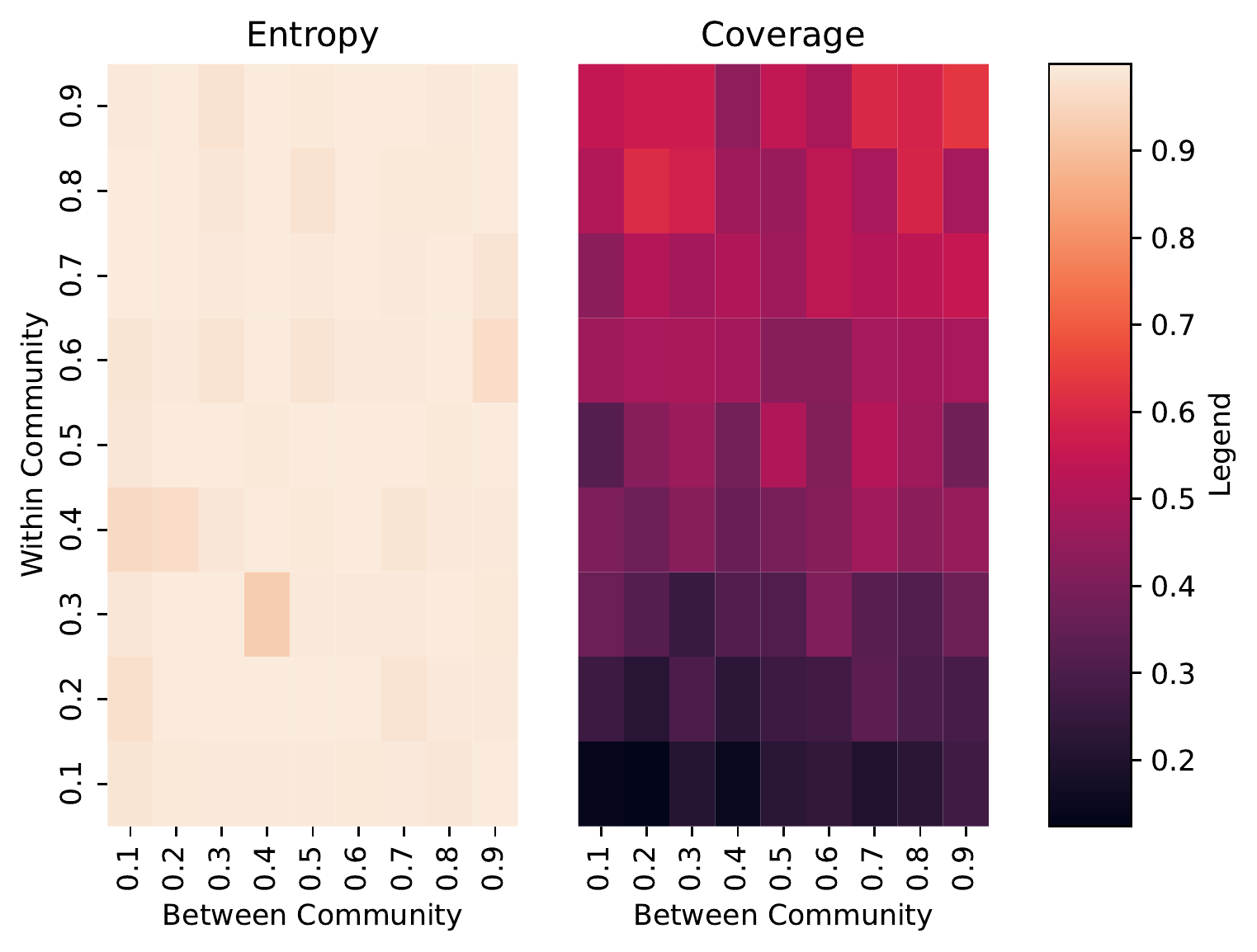}
	\caption{Entropy and coverage (averaged over 50 replications)  on the political blogs network using the proposed seed allocation, as a function of within- and between-community probabilities of transmission.   
	}\label{polblog}
\end{figure}

The second dataset is a Twitter network.   The nodes are members of the Mexican Chamber of Deputies from 2015 to 2018, and an edge indicates that one representative follows the other on twitter \citep{arroyo2022overlapping}.  Again, we ignore the direction of edges for simplicity.  
In the original data, there are nine ground-truth communities corresponding to political parties.
For our simulations, we only included parties with more 30 members, resulting in 350 nodes, 10834 edges, and four communities.
Instead of using parties, we again estimated communities using spectral clustering.   Table \ref{tab:mexican_communities} compares estimated labels with party memberships; they are largely aligned.  
In contrast to the political blogs network, this network has very unbalanced communities.

\begin{table}

\begin{center}
\begin{tabular}{ccccc}
	Estimated &MORENA & PAN & PRD & PRI \\
	Community 1 & 25 & 0 & 0 & 0 \\
	Community 2 & 9 & 0 & 41 & 0 \\
	Community 3 & 0 & 83 & 0 & 0 \\
	Community 4 & 3 & 1 & 3 & 185 \\
\end{tabular}
\end{center}
\caption{Mexican deputies Twitter network:  estimated vs ground truth communities.  }
\label{tab:mexican_communities}
\end{table}

Figure~\ref{mexpol} shows the results of simulated information spread on the Mexican Chamber of Deputies Twitter network, with a budget of 18 seeds.
Again, we varied the probability of transmission within and between communities and used a single time step, and $\lambda=3$ for all simulations.
The results again show high entropy in all cases and coverage increasing with probability of transmission, particularly within communities.  

\begin{figure}[!ht]
	\centering
	\includegraphics[width=0.9\textwidth]{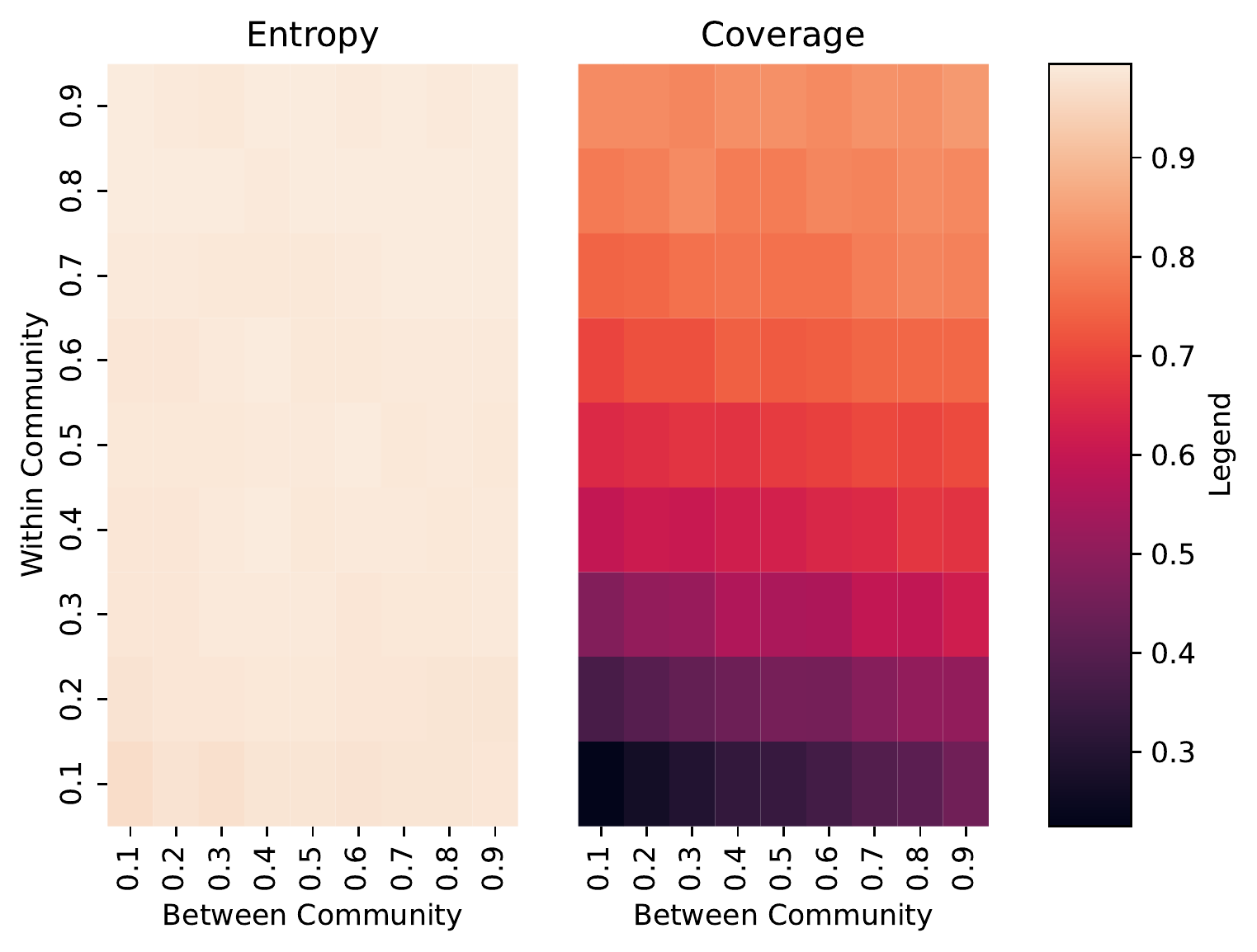}
	\caption{Entropy and coverage (averaged over 50 replications)  on the Mexican deputies Twitter network using the proposed seed allocation, as a function of within- and between-community probabilities of transmission. 	
	}\label{mexpol}
\end{figure}

\section{Discussion}

Finding ways to bring fairness into IM is important because in settings such as public health, disparities in access to information have the potential to exacerbate inequities.  Our main contribution is emphasizing fairness over empirically observed communities rather than some pre-determined demographic covariate.  Fairness across such covariates is likely to be improved by focusing on network communities, since they are often correlated, and allows us to focus on the empirical structure in the network that directly affects information spread.  At the same time, we allow for uncertainties in the observed communities and in the edges themselves by estimating the underlying population communities.

A clear limitation of our approach is the error of the approximation accumulating over time.  A brute force approach to optimization would avoid this problem but greatly increase computational cost, especially for large networks.   Another potential limitation is the need to estimate the number of communities in practice, which is difficult to do when there are many;  however, for the purposes of fairness a coarse partition of the network into a smaller number of communities may well be sufficient.  Finally, exploring more complex network models than the block models may make the algorithm relevant to a broader range of settings;  we leave this direction for future work.  

\section*{Acknowledgements}
O. Mesner was supported by an NSF RTG grant 1646108.
E. Levina's research was partly supported by NSF grants 1916222, 2052918, and  2210439.
J. Zhu's research was supported by NSF DMS 1821243 and NSF DMS 2210439.

\bibliographystyle{apalike}
\bibliography{nf.bib}{}

\end{document}